\documentclass[10pt]{article}

\usepackage{preamble,xcolor}
\usepackage{amsfonts}
\usepackage{comment}
\newcommand{\OPT}{\mathsf{OPT}}
\providecommand{\coloneqq}{\mathrel{\mathop:}=}

\begin{document}

\title{\fontsize{16pt}{19.2pt}\selectfont \bf
   Optimistic Reinforcement Learning with Quantile Objectives
}

\author{\fontsize{12pt}{14.4pt}
    Mohammad Alipour-Vaezi\textsuperscript{a}, Huaiyang Zhong\textsuperscript{a}, Kwok-Leung Tsui\textsuperscript{b}, Sajad Khodadadian\textsuperscript{a,}\thanks{Corresponding Author
    \\Email Addresses: alipourvaezi@vt.edu (M. Alipour-Vaezi); hzhong@vt.edu (H. Zhong); kltsui@vt.edu (K.L. Tsui); sajadk@vt.edu (S. Khodadadian)\\
    ORCID IDs: 0000-0002-7529-1848 (M. Alipour-Vaezi); 0000-0002-2902-1644 (H. Zhong); 0000-0002-0558-2279 (K.L. Tsui); 0000-0002-5197-4652 (S. Khodadadian)}
}

\date{
    \textsuperscript{a}Grado Department of Industrial \& Systems Engineering, Virginia Tech, Blacksburg, VA 24061, USA\\
    \textsuperscript{b}Department of Industrial, Manufacturing, and Systems Engineering, University of Texas at Arlington, Arlington, TX 76019, USA
}

\maketitle

\begin{abstract}
\noindent
 Reinforcement Learning (RL) has achieved tremendous success in recent years. However, the classical foundations of RL do not account for the risk sensitivity of the objective function, which is critical in various fields, including healthcare and finance. A popular approach to incorporate risk sensitivity is to optimize a specific quantile of the cumulative reward distribution. In this paper, we develop \textsc{UCB--QRL}, an optimistic learning algorithm for the $\tau$-quantile objective in finite-horizon Markov decision processes (MDPs). \textsc{UCB--QRL} is an iterative algorithm in which, at each iteration, we first estimate the underlying transition probability and then optimize the quantile value function over a confidence ball around this estimate. We show that \textsc{UCB--QRL} yields a high-probability regret bound $\mathcal O\left((2/\kappa)^{H+1}H\sqrt{SATH\log(2SATH/\delta)}\right)$ in the episodic setting with $S$ states, $A$ actions, $T$ episodes, and $H$ horizons. Here, $\kappa>0$ is a problem-dependent constant that captures the sensitivity of the underlying MDP's quantile value.
\end{abstract}


\spacingset{1.5}

\section{Introduction}
\par Reinforcement learning (RL) provides a general framework for sequential decision making by learning policies through interaction with an unknown environment \parencite{sutton1998reinforcement}. Over the past decade, RL—often coupled with powerful function approximators such as deep neural networks, linear models, splines, and even quantum circuits—has revolutionized our ability to solve complex, high-dimensional decision-making problems. This synergy has enabled RL agents to achieve superhuman performance in games, competitive results in robotic control and locomotion, and large-scale deployment in recommendation systems and operations research \parencite{Elfwing2017Sigmoid-Weighted,busoniu2017reinforcement,Shakya2023Reinforcement,Zhang2020A}. Despite these advances, classical RL methods face an important limitation: they typically optimize expected return and are therefore risk–neutral \parencite{Moos2022Robust}. This highlights the need for formulations that account for the variability and reliability of returns while maintaining data efficiency and principled exploration. \par \emph{Upper Confidence Bound (UCB)} RL operationalizes optimism in the face of uncertainty. At the beginning of each episode, the learner forms confidence regions around the empirical transition model, plans in the most favorable (optimistic) Markov Decision Process (MDP) within those regions, and executes the resulting greedy policy. This plan–act–learn loop achieves near–minimax regret for finite episodic MDPs under the \emph{expectation} objective \parencite{azar2017minimax,auer2008near}.

\par Classical UCB methods are intrinsically \emph{risk–neutral}: they optimize expected return \parencite{liu2020risk}. In safety–critical control, service-level guarantees, and finance, tail performance (e.g., high-percentile delivery time or loss) is paramount \parencite{yang2023safety}. Quantile or Value-at-Risk (VaR)-based objectives capture such requirements directly. The Quantile Markov Decision Process (QMDP) furnishes a backward dynamic program for $\tau$–quantile values, where the quantile operator is defined as
$Q_\tau(X)\coloneqq \inf\{x\in\mathbb R:\ \mathbb P(X\le x)\ge \tau\}$ for the (left-continuous) $\tau$–quantile, 
via an operator on next-step quantile value maps \parencite{li2022quantile}. 
\par A central obstacle in applying the quantile objective in RL is analytical: quantiles are
nonlinear and can change abruptly with small distributional perturbations. In
contrast to expectation operators, the quantile function lacks smoothness and
convexity; therefore, classical optimism analyzes that linearize value differences fail.
We overcome this by (i) expressing the one–step QMDP backup as the quantile of a
\emph{continuation–mixture} random variable, and (ii) proving a sharp, local
Lipschitz property of this mapping under a benign \emph{quantile
margin} at level $\tau$. Informally, if the CDF of the continuation–mixture has
a jump of size at least $\kappa$, then the $\tau$–quantile
is $(2/\kappa)$–Lipschitz with respect to the 1-Wasserstein distance. Coupled
with a TV$\!\to W_1$ bound for mixtures and a coupling argument that uses a
single auxiliary uniform variable to “align’’ next-state randomness, this yields
a clean propagation inequality that is directly analogous to expectation-based
analyzes with explicit $\kappa$–dependence. 
\paragraph{Contributions.}
\begin{enumerate}
\item \textbf{Algorithmic framework.} We introduce \textsc{UCB--QRL}, an optimism-based learning algorithm for quantile objectives in finite-horizon MDPs. Each episode estimates the transition model, builds an $\ell_1$ confidence set, and plans in the most favorable model via a quantile-aware planner.

\item \textbf{High-probability regret guarantees.} We study the convergence of \textsc{UCB--QRL} algorithm and establish a high-probability regret bound of order $\mathcal O\left((2/\kappa)^{H+1}H\sqrt{SATH\log(2SATH/\delta)}\right)$ in the episodic setting, where $\kappa$ is a problem-dependent constant that captures the sensitivity of the underlying MDP's quantile value.
\end{enumerate}

\section{Related Work}
\subsection{Risk-Sensitive MDPs}\label{subsec:rsmdp} 
Risk-sensitive objectives in sequential decision making have been studied under several paradigms. Early work on percentile/quantile criteria analyzed existence, structure, and computation in controlled Markov processes with known dynamics, including shortest-path and service-level formulations \parencite{filar1995percentile,delage2010percentile}. The Quantile MDP (QMDP) framework formalizes dynamic programming for fixed quantile levels and establishes a backward recursion and planning algorithms under known kernels \parencite{li2022quantile}. Closely related, distributional RL propagates the full return distribution and has yielded practical quantile parameterizations such as QR-DQN and IQN \parencite{bellemare2017distributional,dabney2018distributional,dabney2018implicit,Rowland2018An,yang2019fully}. While distributional RL methods typically optimize \emph{expectation}, their estimators provide tools for learning quantile slices.

Beyond quantiles, classical risk-sensitive control optimizes exponential-utility (entropic) criteria, leading to modified Bellman equations and dynamic consistency \parencite{howard1972risk}. Mean–variance MDPs study return–variance trade-offs but face time-inconsistency without special structure \parencite{sobel1982variance,mannor2011mean,guo2012mean}. Coherent risk measures—especially Conditional Value-at-Risk (CVaR) \parencite{rockafellar2000optimization,ROCKAFELLAR20021443}—enable convex surrogates and have been widely explored in RL via value-based, policy-gradient, and actor–critic methods as either objectives or constraints \parencite{chow2014algorithms,tamar2015optimizing,prashanth2014policy}. Constrained MDPs (CMDPs) and safe RL incorporate chance- or CVaR-type constraints using Lagrangian, primal–dual, or Lyapunov approaches \parencite{altman2021constrained,Chow2015Risk-Constrained,Zhang2024CVaR-Constrained,M2022Approximate,Ahmadi2020Constrained}. These lines of research are largely complementary to our setting, which \emph{maximizes a fixed quantile objective} rather than enforcing it as a constraint, and thus requires handling the non-smooth, set-valued nature of the quantile backup itself. Methodologically, quantile regression \parencite{koenker1978regression} underlies many practical estimators used by distributional/quantile RL, but most of this literature does not address online regret with unknown transitions \parencite{dabney2018distributional,dabney2018implicit,yang2019fully}.

\subsection{Optimism and Upper Confidence Bounds (UCB)}\label{subsec:ucb}
Optimism in the face of uncertainty provides near-minimax regret guarantees for expectation-maximizing RL by planning in confidence sets built around empirical transition models. In average-reward communicating MDPs, the UCRL2 algorithm achieves $\tilde{\mathcal O}(D S \sqrt{A T})$-type guarantees via $\ell_1$ confidence sets and Extended Value Iteration (EVI) \parencite{jaksch2010near}. In finite-horizon problems, UCBVI algorithm attains $\tilde{\mathcal O}(H\sqrt{S A T})$ with Hoeffding bonuses and $\tilde{\mathcal O}(\sqrt{H S A T})$ with Bernstein bonuses \parencite{azar2017minimax}. Robust and distributionally robust MDPs planning against uncertainty sets at decision time, yielding max–min or ambiguity-aware backups that are algorithmically akin to optimistic EVI subroutines \parencite{Yu2015Distributionally,Goyal2022Robust,Xu2016Quantile,deo2025design}. 

Adapting optimism to \emph{nonlinear}, tail-focused criteria poses additional challenges: quantile objectives are non-smooth and can change abruptly under small distributional perturbations, so linear value-difference decompositions used for expectation do not directly apply. In one-step settings, bandit studies have designed risk-aware indices for VaR/CVaR and general risk measures \parencite{sani2012risk,galichet2013exploration,cassel2023general}, clarifying how confidence design must reflect tail sensitivity. Extending these ideas to MDPs requires new contraction/sensitivity tools for the backup operator. 

\section{Preliminaries}\label{sec:prelims}
We consider a finite-horizon MDP \(\mathcal{M} = (\mathcal{S},  \mathcal{A}, H, P^\star, r)\), where \(\mathcal{S}\) is the state space, \(\mathcal{A}\) is the action space, \(H\) denotes the horizon, \(P^\star_h(\cdot \mid s,a)\) represents the true transition kernel, and \(r_h(s,a)\in[0,1]\ \) is the reward function for all states \(s\in\mathcal{S}\), actions \(a\in\mathcal{A}\), and horizons $h\in\{0,1,\dots,H-1\}$.
We assume a finite-dimensional state-action space, and we denote $S=|\mathcal{S}|$ and $A=|\mathcal{A}|$.

For a policy $\pi$, transition kernel $P$, horizon step $h\in\{0,\dots,H-1\}$, and  any quantile level $q\in(0,1)$, the $q$--quantile value at state $s$ is defined as
\begin{align}
V_{q,h}^{\pi,P}(s)
\coloneqq 
\inf\!\Bigg\{x\in\mathbb{R}:\;
\mathbb P\!\Biggl(
\sum_{k=h}^{H-1} r_k(S_k, A_k) \le x
\ \Big| 
S_h=s,\ 
A_k=\pi_k(S_k),\ 
S_{k+1}\sim P_k(\cdot\mid S_k,A_k), 
k\geq h \Biggr) \ge q \Bigg\}.
\end{align}
Throughout the paper, we fix a target quantile level $\tau\in(0,1)$. Our learning goal is to solve the following optimization problem
\begin{equation}
\label{eq:opti-goal}
\max_{\pi}\ V^{\pi,P^\star}_{\tau,0}(\bar s),
\end{equation}
where each episode starts from the fixed initial state $\bar s$ (i.e., $S_0^t=\bar s$ for all $t$). We denote by $\pi^\star$ the maximizer in  \eqref{eq:opti-goal} and consider $V^\star_{\tau,0}\equiv V^{\pi^\star,P^\star}_{\tau,0}$.
\section{Finite Horizon UCB--QRL}\label{sec:finite}
\par This section provides a detailed introduction to \textsc{UCB--QRL} algorithm. We first impose an assumption on the underlying MDP in order to control the sensitivity of the quantile backup. 
\begin{definition}[Continuation–mixture]\label{def:Z}
Fix a step $h$, a state–action $(s,a)$, a policy sequence $\pi$, and a transition kernel $P$.
Let $p\coloneqq P_h(\cdot\mid s,a)\in\Delta^S$, and define $p_i\coloneqq P_h(s_i\mid s,a)$. We define the continuation–mixture random variable $Z_{s,a,h}(p;V^{\pi,P}_{\cdot,h+1})$ such that for any $x\in \mathbb{R}$
\begin{align*}
\mathbb P \Big(Z_{s,a,h}\big(p;V^{\pi,P}_{\cdot,h+1}\big)\le x\Big)
=\sum_{i=1}^S p_i\,\phi_i(x),
\qquad
\phi_i(x)\coloneqq \sup \Big\{q\in[0,1]:\, V^{\pi,P}_{q,h+1}(s_i)\leq x\Big\}.
\end{align*}
\end{definition}
\par Intuitively, the continuation–mixture bundles next-state quantile value map into a single scalar random variable whose $q$–quantile equals the one-step backup. This reduction allows us to control quantile sensitivity through transportation distances on distributions rather than directly on set-valued quantile correspondences.
\par Quantile backups are non-smooth and may change discontinuously with small perturbations of the transition law, which prevents the standard expectation-based linearization used in UCB analyzes. To obtain a tractable stability estimate, we assume a jump (“margin”) at the operative quantile of the continuation–mixture (Definition~\ref {def:Z}).

\begin{assumption}[Uniform quantile margin]\label{ass:margin}
For each $(h,s,a)$ and any deterministic policy $\pi$ consider the continuation-mixture
$Z_{s,a,h}\!\bigl(P^\star_h(\cdot\mid s,a); V_{\cdot,h+1}^{\pi,P^\star}\bigr)$. There exists
$\kappa\in(0,1]$ such that \emph{for every} $q\in[0,1]$,
\begin{align*}
\mathbb P \Big(Z_{s,a,h}(P_h^\star(\cdot|s,a); V_{\cdot,h+1}^{\pi,P^\star})\le c^\star_q\Big)
\;-\; \lim_{\epsilon\downarrow 0}\mathbb P \Big(Z_{s,a,h}(P_h^\star(\cdot|s,a); V_{\cdot,h+1}^{\pi,P^\star})\le c^\star_q-\epsilon\Big)
\ \ge\ \kappa.    
\end{align*}
where
$c_q^\star\coloneqq Q_q\!\big(Z_{s,a,h}(P^\star_h(\cdot\mid s,a); V_{\cdot,h+1}^{\pi,P^\star})\big)$.
\end{assumption}
\par Since the state and action sets are finite and $H<\infty$, the collection of deterministic (time–dependent) policies is finite. With deterministic rewards and finite $H$, the continuation–mixture $Z_{s,a,h}\!\bigl(P^\star_h(\cdot\mid s,a); V_{\cdot,h+1}^{\pi,P^\star}\bigr)$ has finite support for every $(s,a,h,\pi)$, so its CDF has a positive jump at each support point. Taking the minimum jump over this finite family yields a uniform margin $\kappa\in(0,1]$, so Assumption~\ref{ass:margin} is satisfied.

\par We begin by fixing a designated start state $\bar s\in\mathcal S$. Each episode $t \in \{0,1,\dots,T-1\}$ starts at $S^t_0 = \bar s$, and within each episode, steps are indexed by $h\in\{0,\ldots,H-1\}$. 
\par Moreover, let $N_h^t(s,a)$ and $N_h^t(s,a,s')$ denote visit and transition counts up to (but excluding) episode $t$.
\par Next, consider a fixed confidence level $\delta\in(0,1)$. Let $T$ denote the number of episodes and $H$ the number of horizons.

We introduce a universal constant
\[
c\ \ge\ \frac{\max\!\left\{2,\sqrt{2\log\!\left( \frac{SATH(2^S-2)}{\delta}\right)}\right\}}
                {\sqrt{\log\!\frac{2SATH}{\delta}}},
\]
and define the confidence radius as
\begin{equation}
\label{eq:confset}
f_\delta(n)=c\sqrt{\frac{\log\!\frac{2SATH}{\delta}}{\max\{1,n\}}}.
\end{equation}

Using this radius, we form an empirical confidence set 
\begin{align*}
\mathcal C^{t}_{\delta}\coloneqq \Bigl\{P:\,
\|P_h(\cdot|s,a)-\widehat P^{t}_h(\cdot|s,a)\|_1 \le f_\delta(N^{t}_h(s,a)), \ \forall s,a,h\Bigr\},
\end{align*}

\par This set contains all transition kernels that are statistically plausible given the data observed up to episode $t$. On the global confidence event $\mathcal{E}_\delta$, which is defined for all $(s,a,h,t)$ as 
\[\mathcal{E}_\delta\!:\ \|P_h^\star(\cdot\mid s,a)-\widehat P_h^{t}(\cdot\mid s,a)\|_1\le f_\delta\!\big(N_h^{t}(s,a)\big),\] 
we have $P^\star\in\mathcal{C}^t_\delta$ simultaneously for all $t$.

\par In parallel, $\forall\,q\in[0,1]$ we define the set of models that respect the quantile margin assumption (Assumption~\ref{ass:margin}):
\begin{align*}
\mathcal C_\kappa \coloneqq \Bigl\{P:\mathbb P\!\Big(Z_{s,a,h}(P_h(\cdot|s,a); V_{\cdot,h+1}^{\pi,P})\le c_{q,s,a,h}^{\pi,P}\Big)
\;-\; \lim_{\epsilon\downarrow 0}\mathbb P\!\Big(Z_{s,a,h}(P_h(\cdot|s,a); V_{\cdot,h+1}^{\pi,P})\le c_{q,s,a,h}^{\pi,P}-\epsilon\Big)
\ \ge\ \kappa, \\
\forall s,a,h, \text{deterministic } \pi \Bigr\},
\end{align*}
where
$c_{q,s,a,h}^{\pi,P}\coloneqq Q_q\!\big(Z_{s,a,h}(P_h(\cdot\mid s,a); V_{\cdot,h+1}^{\pi,P})\big)$.

By construction, the true kernel $P^\star$ lies in $\mathcal C_\kappa$, since it satisfies the margin condition by Assumption~\ref{ass:margin}. With high probability, it also belongs to $\mathcal C^t_\delta$ for all $t$. Thus, $P^\star \in \mathcal C^t_\delta \cap \mathcal C_\kappa$ with high probability.
\par These two sets together form the foundation of our learning algorithm, \textsc{UCB--QRL}, which adopts the optimism-in-the-face-of-uncertainty principle. 
\par At the start of episode $t$, we (i) form an $\ell_1$ confidence region $\mathcal C_\delta^{t+1}$ around the empirical kernel using the radius in Equation~\eqref{eq:confset}; (ii) intersect it with the margin-respecting models $\mathcal C_\kappa$; and (iii) \emph{plan optimistically} over the intersection $\mathcal{\bar C}^{t+1}_{\delta,\kappa} \coloneqq \mathcal C_\delta^{t+1}\cap\mathcal C_\kappa$ to obtain a policy–model pair $(\pi^{t+1},P^{t+1})$ that maximizes the $\tau$–quantile value at the start state. This “estimate $\to$ certify $\to$ plan’’ structure mirrors UCB in expectation-based RL, but the planner is \emph{quantile-aware}.

\begin{algorithm}[H]
\caption{\textsc{UCB--QRL}}
\label{alg:UCB--QRL}
\begin{algorithmic}[1]
\State \textbf{Input:} quantile level $\tau\!\in\!(0,1)$, confidence level $\delta\!\in\!(0,1)$ 
\State \textbf{Initialize:} counts $N^0_{h}(s,a,s')\!\gets\!0$ for all $h$ and $(s,a,s')$; choose any policy $\pi^0$
\For{$t=0,1,\dots,T-1$} 
  \State \textit{Start:} $S_0^{t}\gets \bar s$ \label{alg:UCB--QRL:L1}
  \State \textit{Roll out under $\pi^t$:} generate $\bigl(S_h^{t},A_h^{t},S_{h+1}^{t}\bigr)_{h=0}^{H-1}$
  \State \textit{Update per–step counts:} for each $h$,
  \begin{align*}
      N_h^{t+1}(s,a,s') &=\sum_{i=0}^{t} \mathbbm{1}\!\left\{\, S^i_{h}=s, A^i_{h}=a, S^i_{h+1}=s' \right\}\\
      N_h^{t+1}(s,a) &=\sum_{i=0}^{t} \mathbbm{1}\!\left\{\, S^i_{h}=s, A^i_{h}=a \right\}.
  \end{align*}
    
  \State \textit{Update empirical model:}
  \begin{align*}
      \widehat P^{t+1}_h(s'|s,a)\gets\dfrac{N^{t+1}_h(s,a,s')}{\max\{1,N^{t+1}_h(s,a)\}}
  \end{align*}
  \State \textit{Build confidence sets:} 
         \[\mathcal {\bar C}^{t+1}_{\delta,\kappa}\gets \mathcal{C}^{t+1}_{\delta}\cap \mathcal C_\kappa\]
  \State \textit{Optimistic re-planning} 
        \begin{align*}
        & \pi^{t+1}\ \in\ \arg\max_{\pi}\ \max_{P\in \mathcal {\bar C}^{t+1}_{\delta,\kappa}}\ V^{\pi,P}_{\tau,0}, \quad \text{and set } P^{t+1}\ \text{be any maximizer in } \mathcal {\bar C}^{t+1}_{\delta,\kappa}.
        \end{align*} \label{alg:line9}
\EndFor
\end{algorithmic}
\end{algorithm}

\par In Algorithm~\ref{alg:UCB--QRL}, lines 1–3 set the quantile target and initialize counts and policy. Lines 4–7 roll out one episode under the current policy and update per-step counts and the empirical kernel. Lines 8–10 form the confidence region $\mathcal{\bar C}^{t+1}_{\delta,\kappa}$ and re-plan: the inner maximization over $P\in\mathcal{\bar C}^{t+1}_{\delta,\kappa}$ implements optimism for the \emph{quantile} (not the expectation), while the outer maximization over $\pi$ produces the next policy.
\par Having specified the algorithm, we now turn to its performance analysis. 
\par Our objective is to measure how much reward is lost by following \textsc{UCB--QRL} compared to the optimal $\tau$–quantile policy in the true environment. 
This gap is captured by the notion of \emph{quantile regret}:
\begin{equation}
\label{eq:regret}
\Reg_{\tau}(T)
\;=\;
\sum_{t=0}^{T-1}
\Bigl(
V^{\pi^\star,P^\star}_{\tau,0}(\bar s)-
V^{\pi^{t},P^\star}_{\tau,0}(\bar s)
\Bigr).
\end{equation}
where $\pi^\star$ is the optimal $\tau$–quantile policy under the true kernel $P^\star$. 

Because quantiles can be discontinuous in the underlying distribution, the regret analysis would be more delicate than in the expectation case. To control this, we invoke the margin condition (Assumption~\ref{ass:margin}), which ensures a mild local regularity of the CDF. Under this assumption, we obtain the following high-probability regret bound.
\begin{theorem}[High-probability Quantile-Regret]\label{thm:main}
Let Assumption \ref{ass:margin} hold with parameter $\kappa>0$. Then for \textsc{UCB--QRL} with confidence radii shown in Equation \eqref{eq:confset}, with probability at least $1-2\delta$,
\begin{align}
\label{eq:main_bound}
\Reg_{\tau}(T)
\ \le\
2cH \left(\frac{2}{\kappa}\right)^{H}
\sqrt{SATH\ \log\!\frac{2SATH}{\delta}} \;+\;
\left(\frac{2}{\kappa}\right)^{H+1}\,H\,\sqrt{\frac{2\bigl(1-(\kappa/2)^{2H}\bigr)}{\tfrac{4}{\kappa^2}-1}T\log\!\tfrac{2}{\delta}}. 
\end{align}
\end{theorem}
\section{Discussion and Conclusion}
\par In this section, first, we analyze the regret bound in Theorem~\ref{thm:main} and situate it within the existing literature, emphasizing where our guarantees align with or improve upon prior results in optimistic and risk-sensitive reinforcement learning. Second, we examine the computational profile of Algorithm~\ref{alg:UCB--QRL}, and practical implementation choices. Third, we discuss our underlying assumption and discuss the practical implications of our results.  

\subsection{Computational Aspects}
\par In the setting that we have the knowledge of the transition probability $P^\star$, \textcite{li2022quantile} develops a backward dynamic program that computes $V^{\star}_{\tau,0}(s)$ as follows:

\begin{align}\label{eq:qmdp-bellman}
V_{\tau,h}^{\star}(s)
\;=\; \max_{a\in\mathcal A}\;
\OPT \!\bigl(s,\tau,a; V_{\cdot,{h+1}}^{\star}, P^\star_h(\cdot\mid s,a)\bigr), h=0,\dots,H-1,
\quad V_{\tau,H}^{\star}(s)\equiv 0.
\end{align}
Here, for any distribution $P$, we define
\begin{align}
\OPT(s,\tau,a;V_{\cdot,h+1},P)
= \max_{q\in[0,1]^S}\ \min_{i:\ q_i\neq 1} V_{q_i,h+1}(s_i) \\ \nonumber
\text{s.t.}\quad
\sum_{i=1}^S P(s_i) q_i \le \tau.
\label{eq:opt-primal-def}
\end{align}
\noindent where $s_i\in \mathcal{S}$ denotes the $i$'th state.

\par For non-quantile value functions where we define 
\[ V^{\pi,P}_0(s)=\mathbbm{E}\!\left[\sum_{k=0}^{H-1} r_k(S_k,A_k)\,\middle|\,S_0=s\right],
\] \textcite{jaksch2010near} introduced \emph{Extended Value Iteration} (EVI) in the average-reward setting: at each $(s,a)$, the optimistic kernel is chosen inside the $\ell_1$ confidence slice to maximize the next-step value (effectively transporting probability mass toward higher-value states under the $\ell_1$ budget), and the resulting optimistic MDP is solved by value iteration to obtain the greedy policy. In finite-horizon problems, the analogous optimistic planning pass is the UCBVI backward recursion of \textcite{azar2017minimax}, which can be viewed as an EVI-style update rolled over stages $h=H{-}1,\ldots,0$.

\par In the \textsc{UCB--QRL} algorithm, Line~\ref{alg:line9} computes $\pi^{t+1}$ and selects $P^{t+1}$ as the joint maximizer of $V^{\pi,P}_{\tau,0}$. One can develop a combination of the backward dynamic program developed by \parencite{li2022quantile} and the EVI algorithm developed in \parencite{jaksch2010near} to replace line \ref{alg:line9} of Algorithm \ref{alg:UCB--QRL}. This is a future direction of this research study.

\subsection{Comparison with the Prior Work}
\par Theorem~\ref{thm:main} shows that \textsc{UCB--QRL} achieves high-probability quantile regret that scales as
\[
\Reg_{\tau}(T)\;=\;\tilde{\mathcal O}\!\left(\bigl(\tfrac{2}{\kappa}\bigr)^{\!H}\,H\sqrt{S A T H}\right),
\]
up to logarithmic factors and an additive martingale term specified in Inequality~\eqref{eq:main_bound}. In contrast to risk-neutral UCB results (e.g., $\tilde{\mathcal O}(H\sqrt{S A T})$ for UCBVI with Hoeffding bonuses and $\tilde{\mathcal O}(\sqrt{H S A T})$ with Bernstein bonuses \parencite{azar2017minimax}), and, in average-reward communicating MDPs, $\tilde{\mathcal O}(D S \sqrt{A T})$ for UCRL2 where $D$ is the MDP diameter \parencite{jaksch2010near}. 
\par Our analysis exhibits an explicit sensitivity to the \emph{quantile margin}~$\kappa$. This is unavoidable for quantiles: when the CDF at level~$\tau$ is flat (small~$\kappa$), tiny model errors can shift $Q_\tau$ substantially, and exploration must compensate accordingly.
\par The factor $(2/\kappa)$  arises from iterating a local sensitivity bound across $H$ stages.This reflects worst-case compounding under the margin assumption. Whether this dependence can be improved is open. 
\par In many problems, margins are heterogeneous across stages and states; refined, stagewise analyses that track realized occupancy and local margins can shrink the compounding (e.g., from exponential in $H$ to polynomial or linear in effective horizon), at the expense of heavier notation. Developing such \emph{adaptive-margin} bounds is a promising direction.

\textcite{hau2024q} propose a dynamic-programming decomposition for VaR in MDPs and a \emph{model-free} VaR-Q-learning algorithm that does not assume known transitions and avoids saddle-point solvers. They prove convergence of their algorithm to a unique fixed point induced by a $\kappa$-soft quantile loss; their analysis is presented for finite-horizon, time-indexed control. Our work is \emph{model-based} and provides a nonasymptotic, high-probability \emph{regret} bound for episodic, tabular QMDPs under a quantile margin. \textcite{hau2024q} provide \emph{convergence} guarantees (no regret rates) for a model-free Q-learning scheme tailored to VaR. Methodologically, they define a quantile-aware Q operator and a soft-quantile loss to ensure uniqueness of the fixed point, whereas we construct $\ell_1$ confidence sets and control the quantile backup via the continuation–mixture sensitivity bound (yielding the explicit $(2/\kappa)^H$ dependence). The DP decomposition and risk-indexed operator from \parencite{hau2024q} suggest model-free extensions of \textsc{UCB--QRL}. Conversely, our sensitivity tools and confidence design provide a path toward finite-sample, high-probability guarantees for VaR-style Q-learning under margin conditions. Establishing nonasymptotic regret bounds for model-free VaR control remains open.

\subsection{Assumption Discussion}
\par Assumption~\ref{ass:margin} requires that each continuation–mixture $Z_{s,a,h}$ (Definition~\ref{def:Z}) has a jump of size at least $\kappa$ at the $\tau$-quantile under $P^\star$. In our finite, deterministic-reward setting, these mixtures are discrete, hence a (problem-dependent) $\kappa\!>\!0$ always exists, though it may be small. Analytically, $\kappa$ controls the local Lipschitz constant of the quantile backup with respect to the Wasserstein distance ($W_1$), which is the key to the propagation inequality. Practically, larger margins arise when next-state value distributions allocate non-negligible mass exactly at the operative quantile; small margins indicate intrinsically fragile tails and make tail-optimal learning harder.
\par Our analysis makes a simplifying assumption. We work in the finite tabular setting with episodic horizons, leaving extensions to function approximation (linear, kernelized, or neural) as an open question that will require new concentration tools for distributional value errors. Rewards are assumed deterministic given $(s,a,s')$; when rewards are stochastic, the continuation–mixture can absorb the noise without altering the quantile-optimality structure \parencite{li2022quantile}. Finally, the algorithm itself does not require $\kappa$, but the bound does. Developing data-driven methods to estimate local margins and adapt bonuses accordingly could yield sharper, data-dependent guarantees.

Quantile-optimal learning directly targets tail risk and is natural for safety-critical domains such as service-level guarantees, latency minimization, and adverse-event prevention. The explicit $\kappa$-dependence aligns theory with practice: when the operative quantile is well supported, learning is efficient; when the tail is thin, the bound correctly reflects the increased difficulty. This observation motivates several future directions, including relaxing the quantile target, adopting smoother risk measures such as CVaR, or incorporating problem-specific structure through priors. Beyond this, Bernstein-type confidence bonuses may remove the extra $\sqrt{H}$ factor, and the continuation–mixture framework may extend to other coherent risk objectives; a formal analysis is left to future work. Combining distributional critics with optimism-based exploration in large-scale problems, as well as extending our techniques to infinite-horizon discounted or average-reward settings, remains an important challenge for future work.

We conjecture that Assumption~\ref{ass:margin} is necessary to achieve a sublinear regret for Quantile MDPs. In particular, we believe that if $\kappa$ is not known, and the optimization in line \ref{alg:line9} is performed over $\mathcal{C}_\delta^t$ only, there exists an instance of MDP and $\tau$ for which the Algorithm \ref{alg:UCB--QRL} will cause a linear regret. 

In practice, where $\kappa$ is not known, one can start from an initial $\kappa_0\in (0,1)$ point. As the algorithm runs, if empirical regret appears linear, gradually reduce $\kappa$ until the algorithm starts converging.

\par Finally, two natural extensions are left open.
\begin{enumerate}
    \item \emph{Discounted, infinite-horizon quantile control:} develop an optimistic algorithm and analysis for the discounted objective with $\gamma\in(0,1)$, including a discounted continuation–mixture operator, an appropriate contraction/sensitivity inequality for the quantile backup, and stationary $\ell_1$-confidence sets for model uncertainty.
    \item \emph{Lower bounds for regret:} establish minimax and instance-dependent lower bounds for quantile-regret under margin assumptions to determine which dependencies on $\kappa,S,A,$ and horizon (or effective horizon) are improvable versus unavoidable.
\end{enumerate}

\printbibliography

\section{Appendix}
\subsection{Proof of Theorem~\ref{thm:main}}
\begin{proof}[]
We have
\begin{align}
\Reg_{\tau}(T)
=&
\sum_{t=0}^{T-1}
\Bigl(
V^{\pi^\star,P^\star}_{\tau,0}(\bar s)-
V^{\pi^{t},P^\star}_{\tau,0}(\bar s)
\Bigr)\nonumber\\
=&
\sum_{t=0}^{T-1}
\Bigl(
V^{\pi^\star,P^\star}_{\tau,0}(\bar s)-V^{\pi^t,P^t}_{\tau,0}(\bar s)+V^{\pi^t,P^t}_{\tau,0}(\bar s)-
V^{\pi^{t},P^\star}_{\tau,0}(\bar s)
\Bigr).\tag{Adding and subtracting $V^{\pi^t,P^t}_{\tau,0}(\bar s)$}
\end{align}

By Lemma \ref{lem:optimism-step0} at $s=\bar s$, with probability at least $1-\delta$, for all $t\geq 0$ we have
\[
V^{\pi^\star,P^\star}_{\tau,0}(\bar s)
\le
V^{\pi^t,P^{t}}_{\tau,0}(\bar s),
\]
which implies that with probability at least $1-\delta$,
\[
\Reg_{\tau}(T)
\;\leq \;
\sum_{t=0}^{T-1}
\Bigl(V^{\pi^t,P^t}_{\tau,0}(\bar s)-
V^{\pi^{t},P^\star}_{\tau,0}(\bar s)
\Bigr).
\]

For any policy $\pi$, any state $s$, horizon $h$, any quantile level $q$ and transition probability $P$,  by Lemma \ref{lem:bellman-eval}, the QMDP Bellman recursion gives
\begin{equation}\label{eq:bellman-expand}
V^{\pi,P}_{q,h}(s)
\;=\;
r_h(s,a_h)\;+\;Q_q\!\Big(Z_{s,a_h,h}\big(P_h(\cdot\mid s,a_h);\ V_{\cdot,h+1}^{\pi,P}\big)\Big),
\end{equation}
where $a_h=\pi_h(s)$.

Applying Equation \eqref{eq:bellman-expand} with $P^t$ (optimistic kernel) and $P^\star$ (true kernel),
by Lemma~\ref{lem:bellman-eval} the difference is
\begin{align*}
\Delta_h^t(q)
&\;\coloneqq\;\Big| V^{\pi^t,P^{t}}_{q,h}(S_h^t)-V^{\pi^t,P^\star}_{q,h}(S_h^t)\Big|\\[0.5ex]
&=\;\Big| r_h(S_h^t,A_h^t)
+Q_q\!\Big(Z_{S_h^t,A_h^t,h}(P^t_h(\cdot\mid S_h^t,A_h^t); V^{\pi^t,P^t}_{\cdot,h+1})\Big) - r_h(S_h^t,A_h^t)
-Q_q\!\Big(Z_{S_h^t,A_h^t,h}(P^\star_h(\cdot\mid S_h^t,A_h^t); V^{\pi^t,P^\star}_{\cdot,h+1})\Big)\Big|.
\end{align*}

\noindent
Here, we have $A_h^t=\pi_h^t(S_h^t)$. It's noteworthy that according to Lemma~\ref{lem:deterministic-policy}, $\pi^t$ is a deterministic policy and hence denoting $A_h^t$ as above is appropriate.

Since the immediate reward $r_h(S_h^t,A_h^t)$ does not depend on $P$, it cancels out exactly.
Hence
\begin{equation}
\label{eq:one-step-expanded}
\Delta_h^t(q)
\;=\;\Big|
Q_q\!\Big(Z_{S_h^t,A_h^t,h}\big(P^t_h(\cdot\mid S_h^t,A_h^t);\ V^{\pi^t,P^t}_{\cdot,h+1}\big)\Big)
\;-\;
Q_q\!\Big(Z_{S_h^t,A_h^t,h}\big(P^\star_h(\cdot\mid S_h^t,A_h^t);\ V^{\pi^t,P^\star}_{\cdot,h+1}\big)\Big)\Big|.
\end{equation}

Adding and subtracting  $Q_q\!\Big(Z_{S_h^t,A_h^t,h}\big(P^\star_h(\cdot\mid S_h^t,A_h^t);\ V^{\pi^t,P^t}_{\cdot,h+1}\big)\Big)$, we can write
\begin{align}
\label{eq:eq:one-step-expanded-add-subtract}
\Delta_h^t(q)
&=\Bigg|\underbrace{\Big[
Q_q\!\Big(Z_{S_h^t,A_h^t,h}\big(P^t_h(\cdot\mid S_h^t,A_h^t);\ V^{\pi^t,P^t}_{\cdot,h+1}\big)\Big)
-
Q_q\!\Big(Z_{S_h^t,A_h^t,h}\big(P^\star_h(\cdot\mid S_h^t,A_h^t);\ V^{\pi^t,P^t}_{\cdot,h+1}\big)\Big)
\Big]}_{\text{model term}} \nonumber\\
&\quad+\underbrace{\Big[
Q_q\!\Big(Z_{S_h^t,A_h^t,h}\big(P^\star_h(\cdot\mid S_h^t,A_h^t);\ V^{\pi^t,P^t}_{\cdot,h+1}\big)\Big)
-
Q_q\!\Big(Z_{S_h^t,A_h^t,h}\big(P^\star_h(\cdot\mid S_h^t,A_h^t);\ V^{\pi^t,P^\star}_{\cdot,h+1}\big)\Big)
\Big]}_{\text{propagation term}}\Bigg|.
\end{align}

The \emph{model term} isolates the effect of using $P_h^t$ instead of $P_h^\star$ while freezing the continuation map; by Lemma~\ref{lem:lipschitz},
\[
\text{model term}\ \le\ \frac{H}{\kappa}\,\bigl\|P_h^t(\cdot|S_h^t,A_h^t)-P_h^\star(\cdot|S_h^t,A_h^t)\bigr\|_1.
\]

Let $p^\star\!\coloneqq P_h^\star(\cdot\mid S_h^t,A_h^t)$ and draw $S_{h+1}^t\sim p^\star$.
Write
\begin{equation}\label{eq:prop-start}
\text{propagation term}
\;=\;
\Big| Q_q\!\Big(Z_{S_h^t,A_h^t,h}\big(p^\star;\,V^{\pi^t,P^t}_{\cdot,h+1}\big)\Big)
\;-\;
Q_q\!\Big(Z_{S_h^t,A_h^t,h}\big(p^\star;\,V^{\pi^t,P^\star}_{\cdot,h+1}\big)\Big)\Big|,
\end{equation}

Define the continuation–mixture variables
\[
Z^t\;\coloneqq\;Z_{S_h^t,A_h^t,h}\!\bigl(p^\star;\,V^{\pi^t,P^t}_{\cdot,h+1}\bigr),
\qquad
Z^{\star}\;\coloneqq\;Z_{S_h^t,A_h^t,h}\!\bigl(p^\star;\,V^{\pi^t,P^\star}_{\cdot,h+1}\bigr).
\]

By Definition~\ref{def:Z}, $Z^t$ and \ $Z^{\star}$ have CDF
$\Phi_{p^\star}^{t}(y)=\sum_i p^\star_i\,\phi^{t}_i(y)$
and \ $\Phi_{p^\star}^{\star}(y)=\sum_i p^\star_i\,\phi^{\star}_i(y)$ respectively with
$\phi^{\cdot}_i(y)=\sup\{q\in[0,1]:\,V^{\pi^t,P^{\cdot}}_{q,h+1}(s_i)\le y\}$.
Let $U\sim\mathrm{Unif}[0,1]$ be independent of $S_{h+1}^t$. Let $\mathcal{F}_{h}^{t}$ be the $\sigma$-algebra up to step $h$ of episode $t$. By Lemma~\ref{lem:aux-unif},
$Z^t\ \stackrel{d}{=}\ V^{\pi^t,P^t}_{U,h+1}\!\bigl(S_{h+1}^t\bigr)$ and
$Z^{\star}\ \stackrel{d}{=}\ V^{\pi^t,P^\star}_{U,h+1}\!\bigl(S_{h+1}^t\bigr)$
\emph{conditionally on $\mathcal F_h^t$}. Combining with Equation~\eqref{eq:prop-start} and the quantile form above, we obtain
\begin{equation}\label{eq:prop-as-quantile-diff}
\text{propagation term}
\;=\;
\Big| Q_q\!\Big(V^{\pi^t,P^t}_{U,h+1}(S_{h+1}^t)\Big)
\;-\;
Q_q\!\Big(V^{\pi^t,P^\star}_{U,h+1}(S_{h+1}^t)\Big)\Big|
\qquad\text{(in law, conditionally on $\mathcal F_h^t$)}.
\end{equation}

By Lemma~\ref{lem:quantile-w1}, applied at level $q$ with margin parameter $\kappa>0$ from Assumption~\ref{ass:margin},
\begin{equation}\label{eq:quantile-w1-apply}
\bigl|Q_q(Z^t)-Q_q(Z^{\star})\bigr|
\ \le\ \frac{2}{\kappa}\,
W_1\!\Big(\Law(Z^t\mid\mathcal F_h^t),\ \Law(Z^{\star}\mid\mathcal F_h^t)\Big).
\end{equation}

By the primal (Kantorovich) formulation of $W_1$ \cite{villani2009ot},
\begin{equation}\label{eq:w1-primal}
W_1(\mu,\nu)\;=\;\inf_{\pi\in\Pi(\mu,\nu)} \int |x-y|\,d\pi(x,y),
\end{equation}
so for \emph{any} coupling $\pi$ of the two laws we have
$W_1(\mu,\nu)\le \E_\pi[\,|X-Y|\,]$.
We take the explicit coupling that uses the same randomness $(S_{h+1}^t,U)$ to generate
$X=V^{\pi^t,P^t}_{U,h+1}(S_{h+1}^t)$ and $Y=V^{\pi^t,P^\star}_{U,h+1}(S_{h+1}^t)$
conditionally on $\mathcal F_h^t$, which yields 
\begin{equation}\label{eq:w1-coupling}
W_1(\cdot,\cdot)
\ \le\
\E\!\Big[\,
\big|V^{\pi^t,P^t}_{U,h+1}(S_{h+1}^t)-V^{\pi^t,P^\star}_{U,h+1}(S_{h+1}^t)\big|
\ \Big|\ \mathcal F_h^t\Big].
\end{equation}

Combining Inequalities~\eqref{eq:quantile-w1-apply} and \eqref{eq:w1-coupling} with Equation~\eqref{eq:prop-as-quantile-diff} yields the conditional bound
\begin{align}
\text{propagation term}
\ \leq &\
\frac{2}{\kappa}\,
\E\!\Big[\,
\big|V^{\pi^t,P^t}_{U,h+1}(S_{h+1}^t)-V^{\pi^t,P^\star}_{U,h+1}(S_{h+1}^t)\big|
\ \Big|\ \mathcal F_h^t\Big]\nonumber\\
\leq&\
\frac{2}{\kappa}\,
\E\!\Big[\
\sup_{q'\in[0,1]}\big|V^{\pi^t,P^t}_{q',h+1}(S_{h+1}^t)-V^{\pi^t,P^\star}_{q',h+1}(S_{h+1}^t)\big| \Big|\ \mathcal F_h^t\Big]\label{eq:prop_term_final2}
\end{align}
where the last inequality is due to $U\sim\mathrm{Unif}[0,1]$. 

\noindent
Combining the propagation term with the model-term bound gives, and conditioning on $\mathcal F_h^t$,
\begin{equation}\label{eq:delta-vs-sup}
\Delta_h^t(q)
\;\le\;
\frac{H}{\kappa}\,
\bigl\|P_h^{t}(\cdot\mid S_h^{t},A_h^t)-P^\star_h(\cdot\mid S_h^{t},A_h^t)\bigr\|_1
\;+\;
\frac{2}{\kappa}\,
\E\!\Big[\,
\sup_{q'\in[0,1]}\big|V^{\pi^t,P^t}_{q',h+1}(S_{h+1}^t)-V^{\pi^t,P^\star}_{q',h+1}(S_{h+1}^t)\big|
\ \Big|\ \mathcal F_h^t\Big].
\end{equation}

Define
\[
W_h^t(S_h^t)\coloneqq \sup_{q'\in[0,1]}\Big|V^{\pi^t,P^t}_{q',h+1}(S_{h+1}^t)-V^{\pi^t,P^\star}_{q',h+1}(S_{h+1}^t) \Big|
\]
Taking the supremum over $q$ on the left hand side of Inequality~\eqref{eq:delta-vs-sup}, we get

\begin{equation}\label{eq:W-recursion-raw}
W_h^t(S_h^t)
\;\le\;
\frac{H}{\kappa}\,\bigl\|P_h^{t}(\cdot\mid S_h^{t},A_h^t)-P^\star_h(\cdot\mid S_h^{t},A_h^t)\bigr\|_1
\;+\;
\frac{2}{\kappa}\,
\E\!\Big[\,W_{h+1}^t(S_{h+1}^t)\ \Big|\ \mathcal F_h^t\Big].
\end{equation}

\noindent
Introduce the scaled potential
\[
Y_h^t\ \coloneqq\ \left(\frac{2}{\kappa}\right)^h W_h^t(S_h^t),
\qquad
\xi_{h+1}^t\ \coloneqq\ Y_{h+1}^t-\E\!\big[Y_{h+1}^t\mid\mathcal F_h^t\big].
\]

Then Inequality~\eqref{eq:W-recursion-raw} is equivalent to the \emph{one-step supermartingale} inequality
\begin{equation}\label{eq:Y-supermart}
Y_h^t
\;\le\;
\frac{H}{\kappa}\!\left(\frac{2}{\kappa}\right)^{h}
\bigl\|P_h^{t}(\cdot\mid S_h^{t},A_h^t)-P^\star_h(\cdot\mid S_h^{t},A_h^t)\bigr\|_1
\;+\;
\E\!\big[\,Y_{h+1}^t\mid\mathcal F_h^t\big],
\qquad
|\xi_{h+1}^t|\ \le\ \left(\frac{2}{\kappa}\right)^{h+1}\! H .
\end{equation}

Unrolling Inequality~\eqref{eq:Y-supermart} over $h=0,\dots,H-1$ and using $W_H^t\equiv 0$ (hence $Y_H^t\equiv 0$) yields
\begin{equation}\label{eq:Y-unroll}
Y_0^t
\;\le\;
\frac{H}{\kappa}\sum_{h=0}^{H-1}
\left(\frac{2}{\kappa}\right)^{h}
\bigl\|P_h^{t}(\cdot\mid S_h^{t},A_h^t)-P^\star_h(\cdot\mid S_h^{t},A_h^t)\bigr\|_1
\;-\;\sum_{h=0}^{H-1}\xi_{h+1}^t.
\end{equation}
Therefore, 
\[
W_0^t(\bar s)
\;\le\;\frac{H}{\kappa}\sum_{h=0}^{H-1}
\left(\frac{2}{\kappa}\right)^{h}
\bigl\|P_h^{t}(\cdot\mid S_h^{t},A_h^t)-P^\star_h(\cdot\mid S_h^{t},A_h^t)\bigr\|_1
\;-\;\sum_{h=0}^{H-1}\xi_{h+1}^t.
\]
By a union bound and Lemma~\ref{lem:weissman}, with probability at least $1-\delta$ the confidence events
\begin{equation}\label{eq:l1-balls-fixed}
\bigl\|P^\star_h(\cdot\mid s,a)-\widehat P_h^{t}(\cdot\mid s,a)\bigr\|_1 \le f_\delta\!\bigl(N_h^t(s,a)\bigr),
\qquad
\bigl\|P^{t}_h(\cdot\mid s,a)-\widehat P_h^{t}(\cdot\mid s,a)\bigr\|_1 \le f_\delta\!\bigl(N_h^t(s,a)\bigr)
\end{equation}
hold simultaneously for all $(t,h,s,a)$; hence, by the triangle inequality,
\begin{equation}\label{eq:triangle-fresh}
\bigl\|P^{t}_h(\cdot\mid S_h^{t},A_h^t)-P^\star_h(\cdot\mid S_h^{t},A_h^t)\bigr\|_1
\ \le\ 2\,f_\delta\!\bigl(N_h^t(S_h^{t},A_h^t)\bigr).
\end{equation}
Therefore with probability at least $1-\delta$ one can rewrite the bound of $W_0^t(\bar s)$ as
\[
W_0^t(\bar s)
\;\le\;\frac{H}{\kappa}\sum_{h=0}^{H-1}
\left(\frac{2}{\kappa}\right)^{h}
2\,f_\delta\!\bigl(N_h^t(S_h^{t},A_h^t)\bigr)
\;-\;\sum_{h=0}^{H-1}\xi_{h+1}^t.
\]
Summing over $t=0,\dots,T-1$
\[
\sum_{t=0}^{T-1} W_0^t(\bar s)
\;\le\;\frac{H}{\kappa}\sum_{t=0}^{T-1}\sum_{h=0}^{H-1}
\left(\frac{2}{\kappa}\right)^{h}
2\,f_\delta\!\bigl(N_h^t(S_h^{t},A_h^t)\bigr)
\;-\;\sum_{t=0}^{T-1}\sum_{h=0}^{H-1}\xi_{h+1}^t.
\]

Moreover, $\{\xi_{h+1}^t\}$ is a martingale-difference sequence with $|\xi_{h+1}^t|\ \le\ \left(\frac{2}{\kappa}\right)^{h+1}\! H$, so we apply the \emph{one-sided} Azuma--Hoeffding inequality for martingale differences with non-identical bounds: for all $\lambda>0$, we have
\[
\mathbb P\!\left(-\sum_{t=0}^{T-1}\sum_{h=0}^{H-1} \xi_{h+1}^t \ge \lambda\right)\ \le\ \exp\!\Big(-\tfrac{\lambda^2}{2\sum_{t=0}^{T-1}\sum_{h=0}^{H-1} \left(\left(\frac{2}{\kappa}\right)^{h+1}\! H\right)^2}\Big).
\]
Furthermore, we have
\[
\sum_{t=0}^{T-1}\sum_{h=0}^{H-1} \left(\left(\frac{2}{\kappa}\right)^{h+1}\! H\right)^2
= T\,H^2\sum_{h=0}^{H-1}\Big(\tfrac{2}{\kappa}\Big)^{2h+2}
= T\,H^2\left(\frac{2}{\kappa}\right)^2\frac{(2/\kappa)^{2H}-1}{(4/\kappa^2)-1}
= T\,H^2\Big(\tfrac{2}{\kappa}\Big)^{2H+2}\frac{1-(\kappa/2)^{2H}}{(4/\kappa^2)-1}.
\]

Applying the one-sided inequality with the display above and setting the right-hand side to $\delta$ yields, with probability at least $1-\delta$,

\begin{equation}\label{eq:azuma-fresh}
-\sum_{t=0}^{T-1}\sum_{h=0}^{H-1}\xi_{h+1}^t
\ \le\
\left(\frac{2}{\kappa}\right)^{H+1}\,H\,\sqrt{\frac{2\bigl(1-(\kappa/2)^{2H}\bigr)}{\tfrac{4}{\kappa^2}-1}T\log\!\tfrac{1}{\delta}}.
\end{equation}

Moreover, using Elliptical Potential Lemma \cite[Lemma G.12]{wang2023near}, we have
\[
\sum_{t=0}^{T-1}\sum_{h=0}^{H-1}\frac{1}{\sqrt{N_h^t(S_h^{t},A_h^t)}}\leq \sqrt{SAHT\log T}.
\]
Hence,
\begin{equation}\label{eq:counting-fresh}
\sum_{t=0}^{T-1}\sum_{h=0}^{H-1}
\left(\frac{2}{\kappa}\right)^{h}
f_\delta\!\bigl(N_h^t(S_h^{t},A_h^t)\bigr)
\ \le\
\left(\frac{2}{\kappa}\right)^{H-1}\,
2c\sqrt{SATH\ \log\!\frac{2SATH}{\delta}}.
\end{equation}

Combining Inequalities~\eqref{eq:azuma-fresh} and \eqref{eq:counting-fresh} we have with probability at least $1-\delta$,
\[
\sum_{t=0}^{T-1} W_0^t(\bar s)
\;\le\;\frac{H}{\kappa}\left(\frac{2}{\kappa}\right)^{H-1}\,
2c\sqrt{SATH\ \log\!\frac{2SATH}{\delta}}+ \left(\frac{2}{\kappa}\right)^{H+1}\,H\,\sqrt{\frac{2\bigl(1-(\kappa/2)^{2H}\bigr)}{\tfrac{4}{\kappa^2}-1}T\log\!\tfrac{2}{\delta}}.
\]
Noting $\sum_{t=0}^{T-1}\Bigl(V^{\pi^t,P^t}_{\tau,0}(\bar s)-
V^{\pi^{t},P^\star}_{\tau,0}(\bar s)
\Bigr)\leq\sum_{t=0}^{T-1} W_0^t(\bar s)$, with probability at least $1-\delta$,
\[
\sum_{t=0}^{T-1} \Bigl(V^{\pi^t,P^t}_{\tau,0}(\bar s)-
V^{\pi^{t},P^\star}_{\tau,0}(\bar s)
\Bigr)
\;\le\;\frac{H}{\kappa}\left(\frac{2}{\kappa}\right)^{H-1}\,
2c\sqrt{SATH\ \log\!\frac{2SATH}{\delta}}+ \left(\frac{2}{\kappa}\right)^{H+1}\,H\,\sqrt{\frac{2\bigl(1-(\kappa/2)^{2H}\bigr)}{\tfrac{4}{\kappa^2}-1}T\log\!\tfrac{2}{\delta}}.
\]

Recalling that  with probability at least $1-\delta$, 
$\Reg_\tau(T)\le \sum_{t=0}^{T-1}\Delta_0^t$, using union bound we conclude that with probability at least $1-2\delta$,
\begin{align}
\Reg_{\tau}(T)
\ \le\
2cH \left(\frac{2}{\kappa}\right)^{H}
\sqrt{SATH\ \log\!\frac{2SATH}{\delta}}
\;+\;
\left(\frac{2}{\kappa}\right)^{H+1}\,H\,\sqrt{\frac{2\bigl(1-(\kappa/2)^{2H}\bigr)}{\tfrac{4}{\kappa^2}-1}T\log\!\tfrac{2}{\delta}}.
\end{align}
This completes the proof.
\end{proof}
\subsection{Technical Lemmas}
\begin{lemma}[High-probability optimism at episode $t$]\label{lem:optimism-step0}
Consider Algorithm \ref{alg:UCB--QRL}. For every episode $t$ and state $s$, with probability at least $1-\delta$, we have
\[
V^{\pi^\star,P^\star}_{\tau,0}(s)
\ \le\
V^{\pi^t,P^{t}}_{\tau,0}(s).
\]
\end{lemma}
\begin{proof}[Proof of Lemma \ref{lem:optimism-step0}]
Define the ``good event''
\[
\mathcal E_\delta \;\coloneqq\;
\bigcap_{t=0}^{T-1}\
\bigcap_{h=0}^{H-1}\
\bigcap_{s,a}
\mathcal G_{t,h}(s,a),
\]
where $(t,h,s,a)$, define
\[
\mathcal G_{t,h}(s,a)\;\coloneqq\;\Bigl\{\,\bigl\|P^\star_h(\cdot\mid s,a)-\widehat P^{t}_h(\cdot\mid s,a)\bigr\|_1
\ \le\ f_\delta\!\bigl(N^t_h(s,a)\bigr)\Bigr\}.
\]

By Lemma~\ref{lem:weissman}, for any $\varepsilon>0$,
\[
\mathbb P\!\left(
\bigl\|P^\star(\cdot\mid s,a)-\widehat P^{t}_h(\cdot\mid s,a)\bigr\|_1 \ge \varepsilon
\ \middle|\ N^t_h(s,a)=n\right)
\ \le\ (2^S-2)\,\exp\!\Bigl(-\frac{n\varepsilon^2}{2}\Bigr).
\]
Choosing $\varepsilon=f_\delta(n)=c\sqrt{\frac{\log\!\frac{2SATH}{\delta}}{\max\{1,n\}}}$ (and $n\ge1$),
\[
\mathbb P\bigl(\mathcal G_{t,h}(s,a)^{\mathsf c}\mid N^t_h(s,a)=n\bigr)
\ \le\ (2^S-2)\,\Bigl(\tfrac{2SATH}{\delta}\Bigr)^{-c^2/2}.
\]
For $n=0$, the event $\mathcal G_{t,h}(s,a)$ holds due to $c\geq\frac{2}{\sqrt{\log\frac{2SATH}{\delta}}}$. Hence, unconditionally,
\[
\mathbb P\bigl(\mathcal G_{t,h}(s,a)^{\mathsf c}\bigr)
\ \le\ (2^S-2)\,\Bigl(\tfrac{2SATH}{\delta}\Bigr)^{-c^2/2}.
\]
Applying a union bound over all $t\in\{0,\ldots,T-1\}$, $h\in\{0,\ldots,H-1\}$, $s\in\mathcal S$, $a\in\mathcal A$ gives
\[
\mathbb P\bigl((\mathcal E_\delta)^{\mathsf c}\bigr)
\ \le\ THSA\,(2^S-2)\,\Bigl(\tfrac{2SATH}{\delta}\Bigr)^{-c^2/2}.
\]
Due to $c\geq \sqrt{2 \frac{\log\left(\frac{SATH(2^S-2)}{\delta}\right)}{\log\left(\frac{2SATH}{\delta}\right)}}$, we have 
\[
\mathbb P(\mathcal E_\delta)\ \ge\ 1-\delta.
\]

\noindent Furthermore, on $\mathcal E_\delta$  we have, for every $(t,h,s,a)$,
\[
\bigl\|P^\star_h(\cdot\mid s,a)-\widehat P^{t}_h(\cdot\mid s,a)\bigr\|_1
\ \le\ f_\delta\!\bigl(N^t_h(s,a)\bigr).
\]
Comparing with the confidence set definition in \eqref{eq:confset}, this is exactly the membership condition
$P^\star\in\mathcal {\bar C}^{t}_{\delta,\kappa}$ for all $0\leq t\leq T-1$. Therefore, since the optimistic planner selects $\pi^t$ and $P^{t}\in\mathcal {\bar C}^{t}_{\delta,\kappa}$ to maximize the quantile value,
\[
V^{\pi^\star,P^\star}_{\tau,0}(s)\ \le\ \max_{\pi}\max_{P\in\mathcal {\bar C}^{t}_{\delta,\kappa}}V^{\pi,P}_{\tau,0}(s)
\ =\ V^{\pi^t,P^{t}}_{\tau,0}(s).
\]
This proves the lemma.
\end{proof}

\begin{lemma}[Weissman’s $\ell_1$ concentration]\label{lem:weissman}
Let $X_1,\dots,X_n$ be i.i.d.\ on $[S]\coloneqq\{1,\dots,S\}$ with $\mathbb P(X_1=i)=p_i$. Define the empirical distribution
$\widehat p_i \coloneqq \tfrac1n\sum_{t=1}^n \mathbbm 1\{X_t=i\}$. Then
\[
\mathbb P\!\bigl(\|\widehat p-p\|_1\ge \eps\bigr)\ \le\ (2^S-2)\,\exp\!\left(-\frac{n\eps^2}{2}\right).
\]
\end{lemma}
\begin{proof}[Proof of Lemma \ref{lem:weissman}]
For any $x\in\mathbb R^S$,
\[
\|x\|_1 \;=\; \max_{v\in\{-1,+1\}^S} v^\top x.
\]
If, in addition, $\sum_i x_i=0$, then the maximizers cannot be $v=\mathbf 1$ or $v=-\mathbf 1$
(since $v^\top x=0$ for those two), where $\mathbf 1$ is an all-one vector. Hence
\begin{equation}\label{eq:l1-hypercube}
\|x\|_1 \;=\; \max_{v\in\mathcal V} v^\top x,
\qquad
\mathcal V \;\coloneqq\; \{-1,+1\}^S\setminus\{\mathbf 1,-\mathbf 1\},
\end{equation}
and consequently
\[
\bigl\{x:\|x\|_1 \ge \eps\bigr\}\ \subseteq\ \bigcup_{v\in\mathcal V}\ \bigl\{x:v^\top x \ge \eps\bigr\}.
\]

Fix $v\in\mathcal V$ and define $Y_t^{(v)}\coloneqq v_{X_t}\in\{-1,+1\}$. Then
\[
v^\top \widehat p \;=\; \sum_{i=1}^S v_i \widehat p_i
\;=\; \frac1n \sum_{t=1}^n v_{X_t}
\;=\; \frac1n \sum_{t=1}^n Y_t^{(v)},
\qquad
\mathbb E\,Y_t^{(v)} \;=\; \sum_{i=1}^S p_i v_i \;=\; v^\top p.
\]
Hence
\[
v^\top(\widehat p - p) \;=\; \frac1n \sum_{t=1}^n \bigl(Y_t^{(v)} - \mathbb E Y_t^{(v)}\bigr),
\]
a mean of i.i.d.\ centered random variables taking values in $[-1,1]$.

By Hoeffding’s inequality,
\begin{equation}\label{eq:hoeffding}
\mathbb P\!\left(v^\top(\widehat p - p) \ge \eps\right)
\;\le\; \exp\!\Bigl(-\frac{2 n \eps^2}{(1-(-1))^2}\Bigr)
\;=\; \exp\!\left(-\frac{n\eps^2}{2}\right).
\end{equation}

Combining Equations \eqref{eq:l1-hypercube}, the union bound, and Inequality \eqref{eq:hoeffding},
\[
\mathbb P\!\bigl(\|\widehat p - p\|_1 \ge \eps\bigr)
\;\le\; \sum_{v\in\mathcal V} \mathbb P\!\bigl(v^\top(\widehat p - p) \ge \eps\bigr)
\;\le\; |\mathcal V| \, \exp\!\left(-\frac{n\eps^2}{2}\right).
\]
Since $|\mathcal V|=2^S-2$, the stated bound follows.
\end{proof}

For a probability measure $\mu$ on $\mathbb{R}$, $F_\mu(x)\coloneqq \mu((-\infty,x])$ and the left-continuous quantile function is
$F_\mu^{-1}(u)\;\coloneqq\; \inf\{x\in\mathbb{R}: F_\mu(x)\ge u\}, u\in(0,1]$. We write $Q_\tau(\mu)\equiv F_\mu^{-1}(\tau)$ for the $\tau$-quantile. If $F_\mu$ jumps at $x^\star$, then $F_\mu^{-1}(u)=x^\star$ for all $u\in(F_\mu(x^\star-),F_\mu(x^\star)]$.

\begin{lemma}[Bellman evaluation identity for the $\tau$–quantile]\label{lem:bellman-eval}
Fix a policy $\pi$, a kernel $P$, a stage $h\in\{0,\ldots,H{-}1\}$, and a state $s$.
Let $a\coloneqq \pi_h(s)$ and $p\coloneqq P_h(\cdot\mid s,a)\in\Delta^S$.
Then
\[
V^{\pi,P}_{\tau,h}(s)
\;=\;
r_h(s,a)\;+\;Q_\tau\!\Big(Z_{s,a,h}\big(p;\,V^{\pi,P}_{\cdot,h+1}\big)\Big).
\]
\end{lemma}

\begin{proof}[Proof of Lemma \ref{lem:bellman-eval}]
Fix $h$, $s$, and let $a\coloneqq \pi_h(s)$ and $p\coloneqq P_h(\cdot\mid s,a)\in\Delta^S$.
By definition of the (left–continuous) $\tau$–quantile,
\[
V^{\pi,P}_{\tau,h}(s)
=\inf\!\left\{x:\ \mathbb P\!\Big( \sum_{k=h}^{H-1} r_k(S_k,\pi_k(S_k)) \le x\ \Big|\ S_h=s\Big)\ \ge \tau \right\}.
\]
Write $r_h\coloneqq r_h(s,a)$. Conditioning on the next state $S_{h+1}$ and using the Markov property,
for every $x\in\mathbb R$,
\[
\mathbb P\!\Big( \sum_{k=h}^{H-1} r_k(S_k,\pi_k(S_k)) \le x\ \Big|\ S_h=s\Big)
=\sum_{i=1}^S p_i\;\mathbb P\!\Big(\sum_{k=h+1}^{H-1} r_k(S_k,\pi_k(S_k)) \le x-r_h\ \Big|\ S_{h+1}=s_i\Big),
\]
\noindent where $p_i=P_h(s_i|s,a)$.

Let $F_i$ be the CDF of the $(h{+}1)$–to–$(H{-}1)$ return starting from $s_i$ under $(\pi,P)$:
\[
F_i(t)\;\coloneqq\;\mathbb P\!\Big(\sum_{k=h+1}^{H-1} r_k(S_k,\pi_k(S_k)) \le t\ \Big|\ S_{h+1}=s_i\Big).
\]
By definition of the QMDP quantile map, for each $q\in(0,1)$
\[
V^{\pi,P}_{q,h+1}(s_i)\;=\;F_i^{-1}(q)\quad\text{(left–continuous quantile)}.
\]
By Lemma \ref{lem:right_identity} 
\[F_i(t)=\sup\{q\in[0,1]:\,F_i^{-1}(q)\le t\}.\]
Applying it with $F_i^{-1}(q)=V^{\pi,P}_{q,h+1}(s_i)$ yields
\[
\phi_i(t)\;\coloneqq\;\sup\{q\in[0,1]:\,V^{\pi,P}_{q,h+1}(s_i)\leq t\}
\;=\;\sup\{q\in[0,1]: F_i^{-1}(q)\leq t\}\;=\;F_i(t),
\]
i.e.,
\[
\mathbb P\!\Big(\sum_{k=h+1}^{H-1} r_k(S_k,\pi_k(S_k)) \le t\ \Big|\ S_{h+1}=s_i\Big)
=\phi_i(t)\quad\text{for all }t\in\mathbb R.
\]
Recall $p_i=P_h(s_i\mid s,a)$, hence
\[
\Phi_p(t)\;\coloneqq\;\sum_{i=1}^S p_i\,\phi_i(t)=\sum_{i=1}^S p_i\,F_i(t).
\]

Therefore the CDF of the $h$–step return at $s$ is $x\mapsto \Phi_p(x-r_h)$, and
\[
V^{\pi,P}_{\tau,h}(s)
=\inf\{x:\ \Phi_p(x-r_h)\ge \tau\}
=r_h+\inf\{y:\ \Phi_p(y)\ge \tau\}.
\]
By Definition~\ref{def:Z}, $\Phi_p(\cdot)$ is the CDF of $Z_{s,a,h}\!\big(p;V^{\pi,P}_{\cdot,h+1}\big)$, so
\[
\inf\{y:\ \Phi_p(y)\ge \tau\}
=Q_\tau\!\Big(Z_{s,a,h}\big(p;V^{\pi,P}_{\cdot,h+1}\big)\Big).
\]
Therefore,
\(
V^{\pi,P}_{\tau,h}(s)=r_h(s,a)+Q_\tau\!\Big(Z_{s,a,h}\big(p;V^{\pi,P}_{\cdot,h+1}\big)\Big).
\)
\end{proof}

\begin{lemma}[Right identity for the left–continuous quantile]
\label{lem:right_identity}
Let $F$ be a CDF on $\mathbb R$ and $F^{-1}(q)=\inf\{x:F(x)\ge q\}$ its left–continuous quantile.
For every $t\in\mathbb R$,
\[
\{q\in[0,1]:\,F^{-1}(q)\le t\}=[0,F(t)]\quad\text{and hence}\quad
\sup\{q\in[0,1]:\,F^{-1}(q)\le t\}=F(t).
\]
\end{lemma}

\begin{proof}[Proof of Lemma~\ref{lem:right_identity}]
Define $B_t\coloneqq\{q\in[0,1]:\,F^{-1}(q)\le t\}$. We show the pointwise equivalence
\[
q\in B_t\quad\Longleftrightarrow\quad q\le F(t),
\]
which immediately yields $B_t=[0,F(t)]$ and the stated supremum.

\smallskip
\noindent(\(\Rightarrow\)) Assume $q\in B_t$, i.e., $F^{-1}(q)\le t$. Set $y\coloneqq F^{-1}(q)$.
By definition, $y=\inf\{x:F(x)\ge q\}$. Because $F$ is right–continuous and non-decreasing, the set
$\{x:F(x)\ge q\}$ is right–closed; hence $F(y)\ge q$.\footnote{Equivalently:
for every $\varepsilon>0$, $y+\varepsilon$ belongs to the set, so $F(y+\varepsilon)\ge q$;
right–continuity gives $F(y)=\lim_{\varepsilon\downarrow 0}F(y+\varepsilon)\ge q$.}
Since $y\le t$ and $F$ is non-decreasing, $F(t)\ge F(y)\ge q$, i.e., $q\le F(t)$.

\smallskip
\noindent(\(\Leftarrow\)) Assume $q\le F(t)$. Then $t\in\{x:F(x)\ge q\}$, so this set is nonempty and
$F^{-1}(q)=\inf\{x:F(x)\ge q\}\le t$. Thus $q\in B_t$.

\smallskip
Combining the two directions gives $B_t=[0,F(t)]$. Taking the supremum over $B_t$ yields
$\sup B_t=F(t)$.
\end{proof}

\begin{lemma}[Deterministic optimality for QMDP]\label{lem:deterministic-policy}
Fix a kernel $P$. In the recursion \eqref{eq:qmdp-bellman}, for every $(s,h)$ there exists a \emph{deterministic} maximizer $a^\star\in\arg\max_{a}Q_\tau\!\Big(Z_{s,a,h}\big(P_h(\cdot\mid s,a);\ V^\star_{\cdot,h+1}\big)\Big)$. Consequently, there exists an optimal \emph{deterministic Markov} policy.
\end{lemma}

\begin{proof}[Proof of Lemma~\ref{lem:deterministic-policy}]
Fix a stage $h$, a state $s$, a kernel $P$, and a continuation map $V_{\cdot,h+1}$.
For each action $a\in\mathcal A$ let
\[
F_a(t)\;\coloneqq\;\mathbb P\!\Big(Z_{s,a,h}\big(P_h(\cdot\mid s,a);\ V_{\cdot,h+1}\big)\le t\Big),
\qquad
m_a\;\coloneqq\;Q_\tau\!\Big(Z_{s,a,h}\big(P_h(\cdot\mid s,a);\ V_{\cdot,h+1}\big)\Big).
\]
For any mixed action $\mu\in\Delta^{\mathcal A}$, define the mixture CDF
$F_\mu(t)\coloneqq \sum_{a\in\mathcal A}\mu(a)\,F_a(t)$ (draw $A\sim\mu$, then sample the continuation under $A$).
Let $M\coloneqq \max_{a} m_a$ and fix any $a$. By definition $m_a=\inf\{x:\ F_a(x)\ge \tau\}$, so $M\ge m_a$ implies $F_a(M)\ge \tau$ (monotonicity of $F_a$ suffices; left–continuity is not needed here). Hence
\[
F_\mu(M)\;=\;\sum_a \mu(a) F_a(M)\ \ge\ \sum_a \mu(a)\,\tau\;=\;\tau,
\]
so by the definition of the (left–continuous) $\tau$–quantile, $Q_\tau(F_\mu)\le M=\max_a m_a$.

Proceed by backward induction on $h$. For $h=H$ the claim is trivial. Assume an optimal continuation map $V_{\cdot,h+1}^{\star}$ has been realized at step $h{+}1$ (this is the inductive hypothesis provided by the Bellman program \eqref{eq:qmdp-bellman}). Consider any state $s$ at step $h$. By Lemma~\ref{lem:bellman-eval},
\[
V^{\pi,P}_{\tau,h}(s)=r_h(s,a)\;+\;Q_\tau\!\Big(Z_{s,a,h}\big(P_h(\cdot\mid s,a);\ V_{\cdot,h+1}^{\pi,P}\big)\Big).
\]
Evaluating the optimality backup with continuation fixed at $V_{\cdot,h+1}^{\star}$ amounts to comparing the collection $\{Q_\tau(F_a)\}_{a\in\mathcal A}$ defined in Lemma~\ref{lem:deterministic-policy}. By the argument above, randomization over actions cannot exceed $\max_a Q_\tau(F_a)$, hence a single action $a^\star$ attains the maximum. Define $\pi_h^\star(s)$ to pick such an $a^\star$ (break ties deterministically). Doing this for every state produces a deterministic Markov decision rule at step $h$; composing with the inductively optimal rules from steps $h{+}1,\dots,H{-}1$ yields a deterministic Markov policy that attains the optimal values $V_{\tau,h}^{\star}$. This completes the induction.
\end{proof}

\begin{lemma}[Local Lipschitz of $Z_{s,a,h}(p(\cdot);{V^{\pi,P}_{q,h+1}})$ in $p$]\label{lem:lipschitz}
Fix a step $h$ and state–action $(s,a)$, and let 
$V^{\pi,P}_{q,h+1}:\mathcal S\times[0,1]\to[0,H]$ be the next-step quantile value map. 
Let $Z_{s,a,h}(p(\cdot);{V^{\pi,P}_{q,h+1}})$ be the continuation-mixture variable 
of Definition~\ref{def:Z}. Then for all $P \in \mathcal {\bar C}^{t+1}_{\delta,\kappa}$,
\[
\bigl|Q_\tau\!\Big(Z_{s,a,h}\big(P_h(\cdot|s,a);V^{\pi,P}_{\cdot,h+1}\big)\Big)-Q_\tau\!\Big(Z_{s,a,h}\big(P^\star_h(\cdot |s,a);V^{\pi,P}_{\cdot,h+1}\big)\Big)\bigr|
\ \le\ \frac{H}{\kappa}\,\|P_h(\cdot|s,a)-P^\star_h(\cdot |s,a)\|_1.
\]
with parameter $\kappa>0$.
\end{lemma}

\begin{proof} [Proof of Lemma \ref{lem:lipschitz}]
Let $\mu_{1}$ and $\mu_{2}$ be the laws of $Z_{s,a,h}\big(P_h^\star(\cdot|s,a);V^{\pi,P}_{\cdot,h+1}\big)$ and $Z_{s,a,h}\big(P_h(\cdot|s,a);V^{\pi,P}_{\cdot,h+1}\big)$, respectively.
Moreover, for probability measures
$\mu,\nu$ on $\mathbb R$, Total Variation (TV) and 1-Wasserstein ($W_1$) can be defined as presented in Equations \eqref{eq:TV} and \eqref{eq:W1} respectively \cite{gibbs2002choosing}.
\begin{equation}
\label{eq:TV}
\mathrm{TV}(\mu,\nu) 
= \sup_{A} |\mu(A)-\nu(A)|
= \tfrac12 \sup_{\|f\|_\infty\le1}\Bigl|\int f\,d(\mu-\nu)\Bigr|.
\end{equation}

\begin{equation}
\label{eq:W1}
W_1(\mu,\nu) 
= \sup_{\mathrm{Lip}(f)\le 1}\int f\,d(\mu-\nu)
= \int_0^1 \bigl|F_\mu^{-1}(u)-F_\nu^{-1}(u)\bigr|\,du.
\end{equation}

Since we assume $P \in \mathcal {\bar C}^{t+1}_{\delta,\kappa}$, hence, $P\in \mathcal{C}_\kappa$, which guarantees that $F_2\left(Q_\tau\left(Z_{s,a,h}\big(P_h(\cdot|s,a);V^{\pi,P}_{\cdot,h+1}\big)\right)\right) - F_2\left(Q_\tau\left(Z_{s,a,h}\big(P_h(\cdot|s,a);V^{\pi,P}_{\cdot,h+1}\big)\right)^-\right) \geq \kappa$. So Lemma~\ref{lem:quantile-w1} yields
\[
\bigl|Q_\tau(\mu_{2})-Q_\tau(\mu_{1})\bigr|
\ \le\ \frac{2}{\kappa}\,W_1(\mu_{1},\mu_{2}).
\]
Also, by Lemma~\ref{lem:tv-mixture} and Lemma~\ref{lem:w1-tv-lemma},
\[W_1(\mu_{1},\mu_{2})\le \tfrac{H}{2}\,\|P_h(\cdot|s,a)-P^\star_h(\cdot |s,a)\|_1.\]
Therefore,
\[
\bigl|Q_\tau(\mu_{2})-Q_\tau(\mu_{1})\bigr|\ \le\ \frac{H}{\kappa}\,\|P_h(\cdot|s,a)-P^\star_h(\cdot |s,a)\|_1.
\]
\end{proof}

\begin{lemma}[Quantile sensitivity under a jump margin]\label{lem:quantile-w1}
Let $\mu$ be a probability measure on $[0,H]$ with CDF $F$ and $\tau$-quantile $x^\star:=Q_\tau(\mu)$ such that $F(x^\star)-F(x^{\star-})= \kappa$.
Then for any probability measure $\nu$ on $[0,H]$,
\[
\bigl|Q_\tau(\nu)-Q_\tau(\mu)\bigr|
\;\le\; \frac{2}{\kappa}\, W_1(\mu,\nu).
\]
\end{lemma}
\begin{proof}[Proof of Lemma \ref{lem:quantile-w1}]
Let $F$ and $G$ be the CDFs of $\mu$ and $\nu$, and let
$Q_u(\mu)\coloneqq \inf\{x:F(x)\ge u\}$ and
$Q_u(\nu)\coloneqq \inf\{x:G(x)\ge u\}$ be their left–continuous quantile functions.
By Equation~\eqref{eq:W1}
\[
W_1(\mu,\nu)=\int_0^1 \bigl|Q_u(\mu)-Q_u(\nu)\bigr|\,du.
\]

Let $x^\star\coloneqq Q_\tau(\mu)$. By assumption, $F$ has a jump of size at least $\kappa$ at $x^\star$:
\[
L\coloneqq F(x^{\star-}),\qquad U\coloneqq F(x^\star),\qquad U-L\ge \kappa,\qquad \tau\in(L,U].
\]
Hence
\[
Q_u(\mu)\equiv x^\star\quad \text{for all }u\in(L,U].
\]

Let $v\coloneqq Q_\tau(\nu)$ and $d\coloneqq v-x^\star$. Monotonicity of $Q_\cdot(\nu)$ yields
\[
u\le \tau \ \Rightarrow\ Q_u(\nu)\le v,\qquad
u>\tau \ \Rightarrow\ Q_u(\nu)\ge v.
\]

\emph{Case $d\ge 0$.} For any $u\in(\tau,U]$, since $Q_u(\nu)\ge v\ge x^\star$,
\[
\bigl|x^\star-Q_u(\nu)\bigr|=Q_u(\nu)-x^\star\ \ge\ v-x^\star\ =\ |d|,
\]
so
\[
W_1(\mu,\nu)\ \ge\ \int_{\tau}^{U} \bigl|x^\star-Q_u(\nu)\bigr|\,du\ \ge\ (U-\tau)\,|d|.
\]

\emph{Case $d<0$.} For any $u\in(L,\tau)$, since $Q_u(\nu)\le v\le x^\star$,
\[
\bigl|x^\star-Q_u(\nu)\bigr|=x^\star-Q_u(\nu)\ \ge\ x^\star-v\ =\ |d|,
\]
and therefore
\[
W_1(\mu,\nu)\ \ge\ \int_{L}^{\tau} \bigl|x^\star-Q_u(\nu)\bigr|\,du\ \ge\ (\tau-L)\,|d|.
\]

Combining the two cases gives
\[
W_1(\mu,\nu)\ \ge\ \Big(\mathbf{1}_{\{d\ge 0\}}(U-\tau)+\mathbf{1}_{\{d<0\}}(\tau-L)\Big)\,|d|.
\]
Since $(\tau-L)+(U-\tau)=U-L\ge \kappa$, at least one of the two side lengths is $\ge \kappa/2$.
Thus
\[
W_1(\mu,\nu)\ \ge\ \frac{\kappa}{2}\,|d|
\quad\Rightarrow\quad
|d|\ \le\ \frac{2}{\kappa}\,W_1(\mu,\nu).
\]
Recalling $d=Q_\tau(\nu)-Q_\tau(\mu)$ finishes the proof.
\end{proof}

\begin{lemma}[Mixture total variation vs.\ mixing weights]\label{lem:tv-mixture}
Let $\{\mu_i\}_{i=1}^S$ be probability measures on $[0,H]$. For $p,p'\in\Delta^S$ define the mixtures
$\mu_p=\sum_{i=1}^S p_i\,\mu_i$ and $\mu_{p'}=\sum_{i=1}^S p'_i\,\mu_i$. Then
\[
\mathrm{TV}(\mu_p,\mu_{p'}) \;\le\; \tfrac12\,\|p-p'\|_1.
\]
\end{lemma}

\begin{proof}[Proof of Lemma \ref{lem:tv-mixture}]
Recall Equation~\eqref{eq:TV}: for probability measures $\mu,\nu$,
\[
\mathrm{TV}(\mu,\nu)\;=\;\tfrac12\sup_{\|f\|_\infty\le 1}\Bigl|\int f\,d(\mu-\nu)\Bigr|.
\]
Write the signed measure difference of the two mixtures as
\[
\mu_p-\mu_{p'} \;=\; \sum_{i=1}^S (p_i-p'_i)\,\mu_i .
\]
For any measurable $f$ with $\|f\|_\infty\le 1$,
\[
\Bigl|\int f\,d(\mu_p-\mu_{p'})\Bigr|
\;=\;\Bigl|\sum_{i=1}^S (p_i-p'_i)\int f\,d\mu_i\Bigr|
\;\le\; \sum_{i=1}^S |p_i-p'_i|\,\Bigl|\int f\,d\mu_i\Bigr|.
\]
Since each $\mu_i$ is a probability measure and $\|f\|_\infty\le 1$,
\[
\Bigl|\int f\,d\mu_i\Bigr|\;\le\;\int |f|\,d\mu_i\;\le\;1.
\]
Therefore
\[
\Bigl|\int f\,d(\mu_p-\mu_{p'})\Bigr| \;\le\; \sum_{i=1}^S |p_i-p'_i|.
\]
Taking the supremum over all $\|f\|_\infty\le 1$ and multiplying by $\tfrac12$ yields
\(
\mathrm{TV}(\mu_p,\mu_{p'}) \le \tfrac12 \|p-p'\|_1,
\)
as claimed.
\end{proof}

\begin{lemma}[TV controls $W_1$ on \(\lbrack 0,H \rbrack\)]
\label{lem:w1-tv-lemma}
If $\mu,\nu$ are probability measures supported on $[0,H]$, then
\[
W_1(\mu,\nu)\ \le\ H\,\mathrm{TV}(\mu,\nu).
\]
\end{lemma}

\begin{proof}[Proof of Lemma \ref{lem:w1-tv-lemma}]
Let $\sigma\coloneqq \mu-\nu$ and write its Jordan decomposition
$\sigma=\sigma^-\!+\!(\sigma-\sigma^-)=\sigma^+-\sigma^-$ with
$\sigma^+([0,H])=\sigma^-([0,H])=\mathrm{TV}(\mu,\nu)\eqqcolon m$.
By the Kantorovich–Rubinstein duality (presented in Equation~\eqref{eq:W1}),
\[
W_1(\mu,\nu)=\sup_{\mathrm{Lip}(f)\le 1}\int f\,d\sigma.
\]
For any bounded $f$,
\[
\int f\,d\sigma=\int f\,d\sigma^+-\int f\,d\sigma^-
\le (\sup f)\,m-(\inf f)\,m=(\sup f-\inf f)\,m.
\]
If $\mathrm{Lip}(f)\le 1$ on $[0,H]$, then $\sup f-\inf f\le H$; hence
$W_1(\mu,\nu)\le H\,m=H\,\mathrm{TV}(\mu,\nu)$.
\end{proof}

\begin{lemma}[Auxiliary–uniform representation of the continuation mixture]\label{lem:aux-unif}
Fix episode $t$, step $h$, and the realized pair $(S_h^t,a_h)$ with $a_h=\pi_h^t(S_h^t)$.
Let $p^\star\coloneqq P_h^\star(\cdot\mid S_h^t,a_h)$ and condition on $\mathcal F_h^t$ so that
$S_{h+1}^t\sim p^\star$ is the only randomness going forward at step $h$.
For any kernel $P$ and policy $\pi^t$, define the continuation–mixture variable
\[
Z^{(P)}\ \coloneqq\ Z_{S_h^t,a_h,h}\!\bigl(p^\star;\,V^{\pi^t,P}_{\cdot,h+1}\bigr).
\]
Let $U\sim\mathrm{Unif}[0,1]$ be independent of $\mathcal F_h^t$ and of $S_{h+1}^t$.
Then, conditionally on $\mathcal F_h^t$,
\[
Z^{(P)}\ \stackrel{d}{=}\ V^{\pi^t,P}_{U,h+1}\!\bigl(S_{h+1}^t\bigr).
\]
\end{lemma}

\begin{proof}[Proof of Lemma~\ref{lem:aux-unif}]
Fix $P$ and abbreviate $V_{q,i}\coloneqq V^{\pi^t,P}_{q,h+1}(s_i)$.
For each next state $s_i$, let $F_i$ denote the CDF of the $(h{+}1)$–to–$(H{-}1)$ return under $(\pi^t,P)$
starting from $s_i$. By definition of the QMDP quantile map, $V_{q,i}=F_i^{-1}(q)$ where the inverse is left–continuous.

By Definition~\ref{def:Z}, the CDF of $Z^{(P)}$ (given $\mathcal F_h^t$) is
\[
\Phi_{p^\star}(t)\;=\;\sum_{i} p^\star_i\,\phi_i(t),
\qquad
\phi_i(t)\coloneqq \sup\{q\in[0,1]:\,V_{q,i}\le t\}.
\]
By Lemma~\ref{lem:right_identity}, for every $i$ we have
$\phi_i(t)=\sup\{q: F_i^{-1}(q)\le t\}=F_i(t)$.
Hence
\begin{equation}\label{eq:mixture-cdf}
\Phi_{p^\star}(t)\;=\;\sum_{i} p^\star_i\,F_i(t).
\end{equation}

Now consider $Y^{(P)}\coloneqq V^{\pi^t,P}_{U,h+1}(S_{h+1}^t)=F_{S_{h+1}^t}^{-1}(U)$.
Condition on the event $\{S_{h+1}^t=s_i\}$.
Since $U$ is independent and uniform, Lemma~\ref{lem:right_identity} gives
\[
\mathbb P\!\big(Y^{(P)}\le t\ \big|\ \mathcal F_h^t,\,S_{h+1}^t=s_i\big)
=\mathbb P\!\big(F_i^{-1}(U)\le t\big)
=\mathbb P\!\big(U\le F_i(t)\big)
=F_i(t).
\]
Taking the conditional expectation over $S_{h+1}^t\sim p^\star$ yields
\[
\mathbb P\!\big(Y^{(P)}\le t\ \big|\ \mathcal F_h^t\big)
=\sum_i p^\star_i\,F_i(t)
=\Phi_{p^\star}(t)
=\mathbb P\!\big(Z^{(P)}\le t\ \big|\ \mathcal F_h^t\big),
\]
where we used Equation~\eqref{eq:mixture-cdf} in the second equality and Definition~\ref{def:Z} in the last.
Thus $Y^{(P)}$ and $Z^{(P)}$ have the same conditional CDF given $\mathcal F_h^t$, i.e., the same conditional law. Therefore, conditionally on $\mathcal F_h^t$, $Z^{(P)} \stackrel{d}{=} V^{\pi^t,P}_{U,h+1}(S_{h+1}^t)$, as claimed.
\end{proof}
\end{document}